%% file: main.tex
\documentclass{article}

\usepackage[preprint]{neurips_2023}

\title{Characterizing Tradeoffs in Language Model Decoding with Informational Interpretations}

\author{%
  Chung-Ching Chang \\
  Google Research \\
  \texttt{ccchang@google.com} \\
   \And
   William W. Cohen \\
   Google Deepmind \\
   \texttt{wcohen@google.com} \\
   \AND
   Yun-Hsuan Sung \\
   Google Research \\
   \texttt{yhsung@google.com} \\
}

\usepackage[utf8]{inputenc} % allow utf-8 input
\usepackage[T1]{fontenc}    % use 8-bit T1 fonts
\usepackage{hyperref}       % hyperlinks
\usepackage{url}            % simple URL typesetting
\usepackage{booktabs}       % professional-quality tables
\usepackage{amsfonts}       % blackboard math symbols
\usepackage{nicefrac}       % compact symbols for 1/2, etc.
\usepackage{microtype}      % microtypography

\usepackage{bbm}
\usepackage{tikz}
\usepackage{amsthm}
\usepackage{amsmath} 
\newtheorem{definition}{Definition}[section]
\newtheorem{theorem}{Theorem}[section]

\newtheorem{lemma}{Lemma}[section]

\DeclareMathOperator*{\pmi}{\text{pmi}}

\begin{document}
\maketitle

\begin{abstract}

We propose a theoretical framework for formulating language model decoder algorithms with dynamic programming and information theory. With dynamic programming, we lift the design of decoder algorithms from the logit space to the action-state value function space, and show that the decoding algorithms are consequences of optimizing the action-state value functions. Each component in the action-state value function space has an information theoretical interpretation. With the lifting and interpretation, it becomes evident what the decoder algorithm is optimized for, and hence facilitating the arbitration of the tradeoffs in sensibleness, diversity, and attribution.

\end{abstract}

\input{introduction}
\input{lm_as_dp}
\input{decoding_algorithms}
\input{construction}
\input{interpretation}

\section{Conclusion}

In this work we proposed a theoretical framework for formulating decoder algorithms with dynamic programming and information theory. We first formulated language modeling as a dynamic programming problem. Next, we constructed the action-state value functions for classical and recent sampling algorithms in the literature. Finally, we interpreted the terms in the action-state value functions with information theoretic implications. The framework provides an abstraction of the decoding algorithms design and makes it clear what each algorithm is optimized for. This helps to arbitrate decoder design when tradeoffs are involved.

\begin{ack}
We would like to thank David Reitter, Renat Aksitov, Abdullatif Köksal for many useful discussions in developing the theory. We would also like to thank Tu Vu, Cicero Nogueira dos Santos, and Tania Bedrax-Weiss for useful feedback and for reviewing the paper.
\end{ack}

\bibliographystyle{plainnat}
\bibliography{main.bib} 

\input{appendix}

\end{document}

%% file: introduction.tex
\section{Introduction}

The scaling of large language models (LLMs)~\citep{anil2023palm,openai2023gpt4} has led to the recent qualitative breakthroughs in generative models. On the other hand, serving these LLMs is costly due to their scale. One common approach to reduce the serving cost is to serve a pool of homogeneous models, and each application dynamically adjusts the model for its individual needs through in-context few-shot prompting~\citep{radford2019language}, soft prompting~\citep{lester2021power}, and prefix tuning~\citep{li2021prefixtuning}, instead of finetuning application/task-specific models. These dynamic adjustments are appealing since the underlying models are unmodified, and this allows the traffic of all downstream applications to be pooled and leads to higher hardware utilization. 

Recent works on parameter-efficient finetuning address the issue of dynamically adjusting the model for finetuning. Instead of finetuning all LLM parameters, they freeze LLM parameters, and add and finetune a relatively small number of parameters on top of existing layers, through low-rank approximations~\citep{hu2021lora} or by rescaling inner activations with learned vectors~\citep{liu2022fewshot}. Recent work~\citep{huang2023lorahub} also shows that we can dynamically reuse/swap the tuned parameters between tasks or do mixed-task inference, which makes parameter-efficient finetuning practical in production.

In parallel, there is also a trend for optimizing LLMs with sequence level human/AI guidance, as opposed to predicting only the next token or corrupted spans of text. In Reinforcement Learning Human Feedback (RLHF)~\citep{ouyang2022training}, it first creates a reward model to predict how humans will rate the quality of text generated by the LLM. This reward model is trained on a dataset of human-rated text samples. Next, the reward model is used to finetune the LLM using reinforcement learning. Most RLHF algorithms adapt LLMs through finetuning and, hence, encounter the aforementioned issue of specializing models, a recent work~\citep{santacroce2023efficient} tries to address this issue through parameter-efficient finetuning.

The big question is: can we dynamically adjust LLMs towards sequence level human preference on the fly without modifying the model? Apparently~\citep{santacroce2023efficient} sheds some light on the possibility through parameter-efficient finetuning, but can we extend the possibility to other approaches, such as in-context learning?

In text-to-image generations, the classifier~\citep{dhariwal2021diffusion} and classifier-free~\citep{ho2022classifierfree,saharia2022photorealistic,yu2022scaling} guidance have proven that we can guide model preference dynamically in predictions. In the reverse diffusion process, the gradient of the scoring function is used to gradually convert a Gaussian noise into a realistic image. With classifier guidance, the gradient is taken over both the scoring function and a classifier -- a classifier that guides the denoising process to a specific class, e.g. a dog image. With classifier-free guidance, the gradient is taken over the same scoring function twice, one with evidence in the input, e.g. ``a dog image``, and one without. By contrasting the gradient with or without the evidence, we can significantly improve the image generation towards the preference in the evidence. Some recent work shows that the same methodologies, the classifier guidance and the classifier-free guidance, are applicable to language modeling.

For the classifier guidance, GeDi~\citep{krause2020gedi} shows improvements on topic control and detoxification; critic-guided decoding~\citep{kim2022criticguided} shows improvements on topic control, sentiment control, and detoxification; FUDGE~\citep{Yang_2021} shows improvements on couplet completion in poetry, topic control, and formality change in translations. Controlled decoding~\cite{mudgal2023controlled} shows improvements on dialog safety and response length. Diffusion-LM~\citep{li2022diffusionlm} further extends the method to non-autoregressive language model based on continuous diffusions. All these works demonstrate how to train the classifiers and how to combine the classification scores to guide the decoding towards the preference of the classifiers.

For the classifier-free guidance, context-aware decoding (CAD)~\citep{shi2023trusting} shows improvements in summarization and knowledge conflicting tasks; PREADD~\citep{pei2023preadd} shows improvements in toxic output mitigation, gender bias reduction, and sentiment control. In particular, PREADD showed that the evidence can be an instruction, and you can adjust the guidance scale to positive (negative) value to follow (disobey) the instruction, respectively.

While the dynamic adjustments in prediction shows great improvement in attribution, another line of work showed that there exists tradeoffs. In classical decoding algorithms,~\citep{aksitov2023characterizing} showed that increasing sampling temperature promotes diversity while sacrificing the sensibleness and attributions.~\citep{chang2023kldivergence} showed the tradeoff curves between diversity and attributions (to the evidence) for the classical top-p, top-k, and temperatures sampling. They also proposed a new sampling algorithm to mitigate the tradeoffs. The discovery may be just a tip of the iceberg. Could other dynamic adjustment algorithms also face certain tradeoffs?

To cover different kinds of dynamical adjustment algorithms in the analysis, we first reformulate the decoding algorithms as a dynamic programming (DP) problem, similar to \cite{kim2022criticguided,mudgal2023controlled}, where you can incorporate sequence level preference as a future reward. The classical algorithms that don't care about sequence level preference degenerates the setup. With DP, we lift the decoding algorithm design into a policy optimization problem in the action-state value function space. Surprisingly, it turns out the action-state value function is composed of items with information theoretical interpretations. This makes it clear what each decoding algorithm is optimized for, and is helpful for arbitrating the tradeoffs in design.

The main contribution of this paper is the proposal of a theoretical framework with dynamic programming and information theory to consolidate the synthesis of the decoding algorithms in the action-state value function space. The paper is arranged as follows: In Section \ref{sec:lm_as_dp}, we formulate the autoregressive language model decoding as a DP problem, following the definitions in Appendix \ref{sec:dp_definitions}. In Section \ref{sec:literature_as_dp}, we reformulated several previous works in the proposed action-state value function space, and stated the results as Theorems. For completeness, we also reformulate classical decoding algorithms in Section \ref{sec:classical_sampling_algorithm}. The detailed action-state value function construction for each of the algorithms is illustrated in Section \ref{sec:action_value_function_constructions}. Finally, we provide information theoretical interpretation for each term in the action-state value function, including KL-divergence, entropy, and cross-entropy, that helps to justify what the decoder algorithm is optimized for. We rewrite and summarize the theorems in Table \ref{table:theorems}. The action-state value function formulation makes it easier to understand how a hyperparameter affects the tradeoffs in the generations.

%% file: lm_as_dp.tex
\section{Language Model Autoregressive Decoding as DP}
\label{sec:lm_as_dp}

A generative language model $G$ maps an input sequence $x = \{x_t\}$, optionally with an evidence $e= \{e_t\}$, to the output sequence $y = \{y_t\}$ where $x_t \in \mathcal{V}$, $y_t \in \mathcal{V}$, $e_t \in \mathcal{V}$,
and $\mathcal{V}$ is a set of vocabulary tokens. We formulate decoding the whole output sequence $y = \{y_t\}$ as an episode in DP\footnote{For DP, we follow the notation in~\citep{silver2015}. For convenience, we repeat the definitions in Appendix \ref{sec:dp_definitions}.}. At each decoding step, the decoder observe the state
$s_t = \{e, x, y_{<t}\} \in \mathcal{S}$ and take action $a_t = {y_t} \in \mathcal{A}$. The state transition from one decoding step to the next is an identity function $\mathcal{P}^a_{ss'} = \mathbbm{1}_{s'=s \cup a}$. Without loss of generality, we assume $y_{T-1} = \langle EOS \rangle$ for some $T > 0$. We set the discount factor $\gamma = 1.0$. The reward $r_t \in \mathcal{R}$ are all zero except for $r_T$. The value of the final reward $r_T$ is calculated by a binary discriminator $D:S_{T-1} \times A_{T-1} \rightarrow \{0, 1\}$ where
\[
   r_T = D(s_{T-1}, a_{T-1}) = \begin{cases}1 & \text{if } y \text{ is attributable to the evidence $e$;} \\ 0 & \text{otherwise.} \end{cases}
\]
We use the notation $s^-_t = \{x, y_{<t}\} \in \mathcal{S}$ for the state without the evidence. In Figure \ref{fig:lm_as_dp}, we illustrate the sequence of $s_t$, $a_t$, and $r_t$ in decoding steps.

The discriminator $D$ is optimized for attributions. Please note that $D$ can be optimized for other properties, e.g. safety or politeness. We choose attributions for two reasons: First, it can help to bridge the gap between classifier guidance and classifier-free guidance in the next Section; Second, attribution is transformative if the evidence $e$ is an instruction, instead of a passage. For example, let $e$ be the tokens for ``Please be polite when answering the question``. When the response has high attribution, it should follow the instruction and be polite. Thus, we can easily convert other desired properties into an attribution problem through instructions.

We assume that $D$ is given and the training of $D$ is beyond the scope of this paper. Examples to train $D$ for other tasks than attribution can be found in~\citep{krause2020gedi,kim2022criticguided,Yang_2021,li2022diffusionlm,mudgal2023controlled}.

\begin{figure}
\centering
\begin{tikzpicture}
    \coordinate (s0) at (1, 1);
    \coordinate (a0) at (2, 1);
    \coordinate (s1) at (3, 1);
    \coordinate (r1) at (3, 2);
    \coordinate (a1) at (4, 1);
    \coordinate (s2) at (5, 1);
    \coordinate (r2) at (5, 2);
    \coordinate (atm2) at (7, 1);
    \coordinate (stm1) at (8.3, 1);
    \coordinate (rtm1) at (8.3, 2);
    \coordinate (atm1) at (9.6, 1);
    \coordinate (st) at (10.9, 1);
    \coordinate (rt) at (10.9, 2);
    \node (n_s0) at (s0) {$s_0$};
    \node (n_a0) at (a0) {$a_0$};
    \node (n_s1) at (s1) {$s_1$};
    \node (n_r1) at (r1) {$r_1$};
    \node (n_a1) at (a1) {$a_1$};
    \node (n_s2) at (s2) {$s_2$};
    \node (n_r2) at (r2) {$r_2$};
    \node (n_atm2) at (atm2) {$a_{T-2}$};
    \node (n_stm1) at (stm1) {$s_{T-1}$};
    \node (n_rtm1) at (rtm1) {$r_{T-1}$};    
    \node (n_atm1) at (atm1) {$a_{T-1}$};
    \node (n_st) at (st) {$s_T$};
    \node (n_rt) at (rt) {$r_T$};
    \draw[-latex, shorten <=0.25cm, shorten >=0.25cm] (s0) -- (a0);
    \draw[-latex, shorten <=0.25cm, shorten >=0.25cm] (a0) -- (r1);
    \draw[-latex, shorten <=0.25cm, shorten >=0.25cm] (a0) -- (s1);
    \draw[-latex, shorten <=0.25cm, shorten >=0.25cm] (s1) -- (a1);
    \draw[-latex, shorten <=0.25cm, shorten >=0.25cm] (a1) -- (r2);
    \draw[-latex, shorten <=0.25cm, shorten >=0.25cm] (a1) -- (s2);
    \draw[dotted, shorten <=0.25cm, shorten >=0.4cm] (s2) -- (atm2);
    \draw[-latex, shorten <=0.25cm, shorten >=0.25cm] (atm2) -- (rtm1);
    \draw[-latex, shorten <=0.40cm, shorten >=0.40cm] (atm2) -- (stm1);
    \draw[-latex, shorten <=0.40cm, shorten >=0.40cm] (stm1) -- (atm1);    
    \draw[-latex, shorten <=0.25cm, shorten >=0.25cm] (atm1) -- (rt);
    \draw[-latex, shorten <=0.4cm, shorten >=0.25cm] (atm1) -- (st);
    \node [rotate=-90,text=gray,align=left,text width=2.3cm](n_s0_footnote) at ([shift={(0,-1.4)}]s0.center) {$=\{e, x\}$};
    \node [rotate=-90,text=gray,align=left,text width=2.3cm](n_a0_footnote) at ([shift={(0,-1.4)}]a0.center) {$=y_0$};
    \node [rotate=-90,text=gray,align=left,text width=2.3cm](n_s1_footnote) at ([shift={(0,-1.4)}]s1.center) {$=\{e, x, y_0\}$};
    \node [rotate=-90,text=gray,align=left,text width=2.3cm](n_a1_footnote) at ([shift={(0,-1.4)}]a1.center) {$=y_1$};
    \node [rotate=-90,text=gray,align=left,text width=2.3cm](n_s2_footnote) at ([shift={(0,-1.4)}]s2.center) {$=\{e, x, y_0, y_1\}$};
    \node [rotate=-90,text=gray,align=left,text width=2.3cm](n_atm2_footnote) at ([shift={(0,-1.4)}]atm2.center) {$=y_{T-2}$};
    \node [rotate=-90,text=gray,align=left,text width=3.0cm](n_stm1_footnote) at ([shift={(0,-1.75)}]stm1.center) {$=\{e, x, y_{<T-1}\}$};
    \node [rotate=-90,text=gray,align=left,text width=3.0cm](n_atm1_footnote) at ([shift={(0,-1.75)}]atm1.center) {$=y_{T-1} = \langle EOS \rangle$};
    \node [rotate=-90,text=gray,align=left,text width=3.0cm](n_st_footnote) at ([shift={(0,-1.75)}]st.center)  {$=\{e, x, y_{<T}\}$};
    \node [text=gray,align=left,text width=.6cm](n_r1_footnote) at ([shift={(0.6,0)}]r1.center) {$=0$};
    \node [text=gray,align=left,text width=.6cm](n_r2_footnote) at ([shift={(0.6,0)}]r2.center) {$=0$};
    \node [text=gray,align=left,text width=.6cm](n_rtm1_footnote) at ([shift={(0.8,0)}]rtm1.center) {$=0$};
\end{tikzpicture}
\caption{Language Model Decoding as DP}
\label{fig:lm_as_dp}
\end{figure}
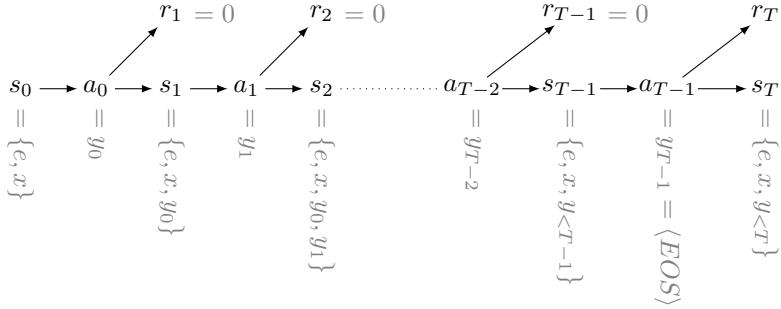

%% file: decoding_algorithms.tex
\section{Decoding Algorithms in the DP framework}

For novel decoding algorithms in the literature~\citep{shi2023trusting,pei2023preadd,chang2023kldivergence} and the classical greedy and temperature sampling algorithms, we formulate the action-state value functions and the corresponding decoding algorithms in theorem statements. We leave the construction of the action-state value functions in Section \ref{sec:action_value_function_constructions}, the derivation of decoding algorithms from action-state value functions in appendices, and the information theoretical interpretation in Section \ref{sec:interpretation}.

\subsection{Literature in the DP framework}
\label{sec:literature_as_dp}

We compare three different decoding policies in the literature: 
\begin{enumerate}
    \item Classifier guidance decoding~\citep{krause2020gedi,kim2022criticguided,Yang_2021,li2022diffusionlm,mudgal2023controlled}.
    \item Classifier-free guidance decoding~\citep{shi2023trusting,pei2023preadd}.
    \item KL-divergence guided temperature sampling~\citep{chang2023kldivergence}.
\end{enumerate}
In the first two items, the cited papers do not use the name with \textit{classifier}. We borrow the concept from the generative text-to-image diffusion models where the classifier guidance decoding refers to a generation guided by a classifier/discriminator~\citep{dhariwal2021diffusion} while the classifier free guidance decoding refers to a generation guided by two distributions: one with evidence in the input and one without~\citep{ho2022classifierfree,saharia2022photorealistic,yu2022scaling}. The cited papers~\citep{krause2020gedi,kim2022criticguided,Yang_2021,li2022diffusionlm,shi2023trusting,pei2023preadd,mudgal2023controlled} applies the concept to language modeling.

For each of the theorems, we formulate the action-state value function and the corresponding explicit optimal policy. For the background in DP, please review Appendix \ref{sec:dp_definitions} for the definitions and notations. Additionally, 
\begin{itemize}
    \item $\mathbb{P}_G(\cdot|s_t)$, shorthanded as $\mathbb{P}_G$, stands for the pretrained language model that takes state $s_t$ as input and outputs the next token distribution (without temperature scaling);
    \item $\mathbb{P}_\pi(r_T=1 | s_t)$ stands for the probability of $r_T = 1$ given the current state $s_t$, assuming all future DP steps are following policy $\pi$ for decoding.
\end{itemize}
We also use information theoretical notations:
\begin{itemize}
    \item KL-divergence $D_{KL}(\pi||\mathbb{P}_G)$ is a shorthand for $D_{KL}(\pi(\cdot|s_t)||\mathbb{P}_G(\cdot|s_t))$;
    \item KL-divergence $D_{KL}(\mathbb{P}_G||\mathbb{P}_G^-)$ is a shorthand for $D_{KL}\left(\mathbb{P}_G(\cdot|s_t)||\mathbb{P}_G(\cdot|s^-_t)\right)$;
    \item Entropy $H(\pi)$ is a shorthand for $H(\pi(\cdot|s_t))$;
    \item Cross entropy $H(\pi||\mathbb{P}_G)$ is a shorthand for $H(\pi(\cdot|s_t)||\mathbb{P}_G(\cdot|s_t))$.
\end{itemize}
Finally, we use $\mathbb{R}_{>0}$ for the set of positive real numbers, and $\mathbb{E}$ for the expectation.

\begin{theorem}[Classifier guidance decoding in the DP Framework] \label{theorem:classifier_guidance}
The decoding algorithm in the classifier guidance decoding is a stochastically optimal policy in the DP framework as follows: 
\begin{align*}
    \pi_*(a_t|s_t) &= \underset{\pi}{\arg\max} \left\{ \lambda \mathbb{E}_{\pi(\cdot|s_t)}\left(\log\frac{\mathbb{P}_\pi(r_T=1 | s_t, \cdot)}{\mathbb{P}_\pi(r_T=1 | s_t)}\right) - D_{KL}\left(\pi||\mathbb{P}_G\right) \right\} \\
    &\propto \left(\frac{\mathbb{P}_*(r_T=1 | s_t, a_t)}{\mathbb{P}_*(r_T=1 | s_t)}\right)^\lambda \mathbb{P}_G(a_t|s_t)
\end{align*}
where $\lambda \in \mathbb{R}_{>0}$ is a hyperparameter.
\end{theorem}

Please note that in Theorem \ref{theorem:classifier_guidance} we can drop the denominator $\mathbb{P}_*(r_T=1 | s_t)$ in the formula as it is a constant given $s_t$. We keep it in the Theorem to be parallel to Theorem \ref{theorem:classifier_free}. Furthermore, we can rewrite the fraction in the optimal policy as a function of the state value function and the action-state value function, as suggested in Lemma~\ref{lemma:binary_discriminator} and similar to~\citep{mudgal2023controlled}.

\begin{theorem}[Classifier-free guidance decoding in the DP Framework] \label{theorem:classifier_free}
The decoding algorithm in the classifier-free guidance decoding is a stochastically optimal policy in the DP framework as follows: 
\begin{align*}
    \pi_*(a_t|s_t) &= \underset{\pi}{\arg\max} \left\{ \lambda \mathbb{E}_{\pi(\cdot|s_t)}\left(\log\frac{\mathbb{P}_G(\cdot|s_t)}{\mathbb{P}_G(\cdot|s_t^-)}\right) - D_{KL}\left(\pi||\mathbb{P}_G\right) \right\} \\
    &\propto \left(\frac{\mathbb{P}_G(a_t|s_t)}{\mathbb{P}_G(a_t|s_t^-)}\right)^\lambda \mathbb{P}_G(a_t|s_t)
\end{align*}
where $\lambda \in \mathbb{R}_{>0}$ is a hyperparameter.
\end{theorem}

\begin{theorem}[KL-divergence guided temperature sampling in the DP Framework] \label{theorem:policy_kl}
The decoding algorithm in KL-divergence guided temperature sampling is a stochastically optimal policy in the DP framework as follows: 
\begin{align*}
    \pi_*(a_t|s_t) &= \underset{\pi}{\arg\max} -\Bigl\{
    (\lambda(s_t, s_t^-)+1) D_{KL}\left(\pi||\mathbb{P}_G\right) + \lambda(s_t, s_t^-) H(\pi) \Bigr\} \\
    &\propto \mathbb{P}_G(a_t|s_t)^{\lambda(s_t, s_t^-) + 1}
\end{align*}
where $H$ is the entropy, $\lambda(s_t, s_t^-) = h\left(D_{KL}\left(\mathbb{P}_G||\mathbb{P}_G^-\right)\right)$ and $h: \mathbb{R}_{\ge 0} \rightarrow \mathbb{R}_{\ge 0}$ is any monotonically increasing function.
\end{theorem}

As a side note, in the original paper~\citep{chang2023kldivergence}, $h(x) = 2^{\frac{x}{\sigma}} - 1$ for some hyperparameter $\sigma \in \mathbb{R}_{>0}$. With this convention, the optimal policy
\[
  \pi_*(a_t|s_t) \propto \mathbb{P}_G(a_t|s_t)^{f\left(D_{KL}\left(\mathbb{P}_G||\mathbb{P}_G^-\right)\right)^{-1}}
\]
where $f(x) = (h(x) + 1)^{-1} = 0.5^{\frac{x}{\sigma}}$ is the dynamic temperature applied to $\mathbb{P}_G$ in each decoding step.

In Theorem \ref{theorem:policy_kl}, we modified the temperature sampling algorithm so that we can reuse $\mathbb{P}_G$ in the formulation. Instead of applying temperature $T$ to the logits with one softmax function, we cascade two softmax functions, the first with temperature $1.0$ (a.k.a. $\mathbb{P}_G$) and the second with temperature $T$ as a function of $f(\cdot)$. This modification should not change the result meaningfully. In fact, the modification has a better information theoretic interpretation.

\subsection{Classical Sampling Algorithms in the DP Framework} \label{sec:classical_sampling_algorithm}

We consolidate the classical temperature sampling and greedy algorithm into the framework by viewing them as a degenerated DP problem where we don't care about the future reward.

\begin{theorem}[Temperature sampling in the DP Framework] \label{theorem:policy_t}
The decoding algorithm in the conventional temperature sampling is a stochastically optimal policy in the DP framework as follows: 
\begin{align*}
    \pi_*(a_t|s_t) &= \underset{\pi}{\arg\max}-\left\{
    \frac{1}{T} \cdot D_{KL}\left(\pi||\mathbb{P}_G\right) + \left(\frac{1}{T} - 1\right) H(\pi) \right\} \\
    &= \underset{\pi}{\arg\max}-\frac{1}{T}\bigg\{
    T \cdot D_{KL}\left(\pi||\mathbb{P}_G\right) + \left(1-T\right) H(\pi, \mathbb{P}_G) \bigg\} \\
    &\propto \mathbb{P}_G(a_t|s_t)^\frac{1}{T}
\end{align*}
where $T \in \mathbb{R}_{>0}$ is the temperature.
\end{theorem}
\begin{proof}
Follow the proof in Theorem \ref{theorem:policy_kl} by setting static $\lambda = \frac{1}{T} - 1$ and the fact that $H(\pi, \mathbb{P}_G) = D_{KL}\left(\pi||\mathbb{P}_G\right) + H(\pi)$.
\end{proof}

In Theorem \ref{theorem:policy_t}, please note that when we slide $T$ from $1.0$ to $0.0$ in the second equation, we are interpolating the terms in the curly brackets from $D_{KL}\left(\pi||\mathbb{P}_G\right)$ to $H(\pi, \mathbb{P}_G)$.

\begin{theorem}[Greedy algorithm in the DP Framework] \label{theorem:policy_greedy}
The greedy algorithm is a deterministic optimal policy in the DP framework as follows: 
\begin{align*}
    \pi_*(a_t|s_t) &= \underset{\pi}{\arg\max}{-H(\pi, \mathbb{P}_G)}
    = \underset{\pi}{\arg\max} \; \mathbb{E}_\pi \log \mathbb{P}_G(a_t|s_t) \\
    &= \begin{cases}1 & \text{if } a_t = \underset{a}{\arg\max} \; \mathbb{P}_G(a|s_t); \\ 0 & \text{otherwise.} \end{cases}
\end{align*}
\end{theorem}
\begin{proof}
Follow the proof in Theorem \ref{theorem:policy_t} by taking $\lim_{T \to 0}$.
\end{proof}

For simplicity, we can present the deterministic optimal policy in Theorem \ref{theorem:policy_greedy} with an indicator function as $\mathbbm{1}_{a_t = \underset{a}{\arg\max} \; \mathbb{P}_G}$.

%% file: construction.tex
\section{Constructing Action-State Value Functions}
\label{sec:action_value_function_constructions}

In the previous theorems, we stated the action-state value function and the corresponding optimal policy. In this section, we will learn how to construct the action-state value function.
In the DP framework, we start by assuming that the reward $r_t \in \mathcal{R}$ are all zero except for the final binary $r_T$. A binary $r_T$ leads to an action-state value function being a pointwise mutual information\footnote{See Appendix \ref{sec:define_pmi} for the definition.} $\pmi(r_T=1, a_t|s_t)$.

\begin{lemma}[State and Action Values for Binary Discriminators] \label{lemma:binary_discriminator}
For a DP setup with binary $r_T$ and $r_t = 0$ for all $t < T$. For any policy $\pi$, the state value and action values are given by
\begin{align*}
    V_\pi(s_t) &= \mathbb{P}_\pi(r_T=1 | s_t) \\
    Q_\pi(s_t) &= \mathbb{P}_\pi(r_T=1 | s_t, a_t) 
\end{align*}
and the optimal deterministic policy is given by
\[
    a_* = \underset{a \in \mathcal{A}}{\arg\max} \; \log\frac{\mathbb{P}_*(r_T=1 | s_t, a_t)}{\mathbb{P}_*(r_T=1 | s_t)}.
\]
\end{lemma}
\begin{proof}
    By definition,
    \begin{align*}
        V_\pi(s_t) &= \mathbb{E}(G_t | S_t = s_t) \\
        &= \mathbb{E}(R_T | S_t = s_t) \\
        &= 1 \cdot \mathbb{P}_\pi(r_T=1 | S_t = s_t) + 0 \cdot \mathbb{P}_\pi(r_T=1 | S_t = s_t) \\
        &= \mathbb{P}_\pi(r_T=1 | S_t = s_t)
    \end{align*}
    Similarly for $Q_\pi(s_t, a_t)$. By definition, the optimal policy $a_*$ is given by
    \begin{align*}
        a_* &= \underset{a_t \in \mathcal{A}}{\arg\max} \; Q_*(s_t, a_t) \\
        &= \underset{a_t \in \mathcal{A}}{\arg\max} \;\frac{Q_*(s_t, a_t)}{V_*(s_t)} \\
        &= \underset{a_t \in \mathcal{A}}{\arg\max} \;\frac{\mathbb{P}_*(r_T=1 | s_t, a_t)}{\mathbb{P}_*(r_T=1 | s_t)} \\
        &= \underset{a_t \in \mathcal{A}}{\arg\max} \;\log\frac{\mathbb{P}_*(r_T=1 | s_t, a_t)}{\mathbb{P}_*(r_T=1 | s_t)}
    \end{align*}
    The second equation hold because $V_*(s_t)$ is a constant once $s_t$ is given.
\end{proof}

In other words, for a DP setup with only binary $r_T$, the optimal policy is to select the action that is most correlated with $r_T = 1$.

\subsection{Construction for Theorem \ref{theorem:classifier_guidance}}

The optimal policy in Lemma \ref{lemma:binary_discriminator} is a deterministic policy. For most LLM decoding, we use  sampling, instead of greedy sampling, to increase response diversity and avoid repetitions. With sampling, we sample $a_t$ according to the policy $\pi(\cdot|s_t)$. Therefore, the optimal stochastic policy for temperature sampling should maximize the expectation over the distribution of $\pi$.
\begin{equation*}
    \pi_* = \underset{\pi}{\arg\max} \left\{\mathbb{E}_{\pi(\cdot|s_t)}\left(\log\frac{\mathbb{P}_\pi(r_T=1 | s_t, \cdot)}{\mathbb{P}_\pi(r_T=1 | s_t)}\right) \right\}
\end{equation*}
This stochastic policy is guided by the discriminator $D$ and is optimized for attributions. In practice, the policy should be optimized for both attributions and sensibleness. Since the pretrained generative model $\mathbb{P}_G$ is optimized for sensibleness, we can add a reward $-D_{KL}\left(\pi(\cdot|s_t)||\mathbb{P}_G(\cdot|s_t)\right)$ to prevent the policy $\pi$ drifting away from $\mathbb{P}_G$. As a result, the overall reward is the weighted sum of two terms:
\begin{equation} \label{eq:optimal_policy_with_sampling_and_classifier_and_kl}
    \pi_* = \underset{\pi}{\arg\max} \left\{ \lambda \mathbb{E}_{\pi(\cdot|s_t)}\left(\log\frac{\mathbb{P}_\pi(r_T=1 | s_t, \cdot)}{\mathbb{P}_\pi(r_T=1 | s_t)}\right) - D_{KL}\left(\pi||\mathbb{P}_G\right) \right\}
\end{equation}

This completes the construction of the action-state value for Theorem \ref{theorem:classifier_guidance}. For the derivation of the optimal policy $\pi_*$, please refer to the Appendix \ref{sec:proof_of_theorem}.

\subsection{Construction for Theorem \ref{theorem:classifier_free}}

For Theorem \ref{theorem:classifier_free}, we replace all appearance of $\log\frac{\mathbb{P}_*(r_T=1 | s_t, a_t)}{\mathbb{P}_*(r_T=1 | s_t)}$ by $\log\frac{\mathbb{P}_G(a_t | s_t)}{\mathbb{P}_G(a_t | s^-_t)}$ in Theorem \ref{theorem:classifier_guidance} and Appendix \ref{sec:proof_of_theorem_policy_discriminator}, and this completes the proof. We have two interpretations for this replacement.

In the first interpretation, the former and latter terms can be rewritten as $\pmi(r_T=1, a_t|s_t)$ and $\pmi(e, a_t|s_t^-)$, respectively. Effectively, the former calculates how the current action $a_t$ affects $y$, once fully decoded, is attributable to the evidence $e$ as suggested by the discriminator $D$. The latter, on the other hand, calculates how the current action $a_t$ correlates to the evidence $e$ directly as suggested by the generator $\mathbb{P}_G$.

In other words, both approaches calculate how $a_t$ improves attributions, either through the discriminator $D$ or through the generator $\mathbb{P}_G$. Both approaches are valid and it is hard to say which is better without any additional information.

In the second interpretation, we rewrite equation (\ref{eq:optimal_policy_with_sampling_and_classifier_and_kl}) with the replacement and approximate $\mathbb{P}_G$ by the policy $\pi$ as suggested by the second term in the right hand side,
\begin{align*}
    \pi_* &= \underset{\pi}{\arg\max} \left\{ \lambda \mathbb{E}_{\pi(\cdot|s_t)}\left(\log\frac{\mathbb{P}_G(\cdot | s_t)}{\mathbb{P}_G(\cdot | s^-_t)}\right) - D_{KL}\left(\pi||\mathbb{P}_G\right) \right\} \\
    &\approx \underset{\pi}{\arg\max} \left\{ \lambda \mathbb{E}_{\pi(\cdot|s_t)}\left(\log\frac{\pi(\cdot | s_t)}{\pi(\cdot | s^-_t)}\right) - D_{KL}\left(\pi||\mathbb{P}_G\right) \right\} \\
    &= \underset{\pi}{\arg\max} \left\{\lambda \cdot D_{KL}\left(\pi(\cdot | s_t)||\pi(\cdot | s_t^-)\right) - D_{KL}\left(\pi||\mathbb{P}_G\right)\right\}
\end{align*}
In other words, the optimal policy is a tradeoff between two terms, the first term promotes its ability to correlate the response with the evidence $e$, and the second term prevents it to drift away from the pretrained model $\mathbb{P}_G$.

\subsection{Construction for Theorem \ref{theorem:policy_kl}}

The construction of the action-state value function for Theorem \ref{theorem:policy_kl} is evident. The action-state value function is a weighted sum of two terms: $-D_{KL}\left(\pi||\mathbb{P}_G\right)$ prevents the optimal policy from drifting away from the pretrained policy $\mathbb{P}_G$; $-H(\pi)$ reduces entropy or diversity.

For simplicity, we will first ignore $h$ in $\lambda(s_t, s_t^-)$ for now and assume $\lambda(s_t, s_t^-) = D_{KL}\left(\mathbb{P}_G||\mathbb{P}_G^-\right)$. The weight $\lambda$ is adjusted based on how relevant is this decoding step to the presence of the evidence~$e$.
\begin{itemize}
    \item If $D_{KL}\left(\mathbb{P}_G||\mathbb{P}_G^-\right)$ is small, the presence of the evidence is irrelevant to the token distribution in the current decoding step, we optimize for the first term that encourages the policy to be close to the pretrained model $\mathbb{P}_G$;
    \item If $D_{KL}\left(\mathbb{P}_G||\mathbb{P}_G^-\right)$ is large, the presence of the evidence matters, so we optimize the policy for both terms: close to $\mathbb{P}_G$ (by the first term) and not evenly distributed (by the second term). This uneven distribution turns out to be that with lower temperature.
\end{itemize}

Applying the monotonically increasing function $h$ reshape the KL-divergence. In the original work~\citep{chang2023kldivergence} with $h(x) = 2^{\frac{x}{\sigma}} - 1$, the function $h$ avoid penalizing $x$ when $x < \sigma$ but apply significant penalty when $x > \sigma$. In other words, the function $h$ is just a convenient utility to reshape the KL-divergence with a single hyperparameter $\sigma$ that defines the threshold.

%% file: interpretation.tex
\section{Informational Interpretations} \label{sec:interpretation}

In constructing the action-state value functions, we already introduced several information-theoretic terms in the formulations, including KL-divergence, entropy, and cross-entropy. These terms are highly correlated to the property of the resulting decoder algorithms, including sensibleness, attribution, and diversity. Let's first inspect each component and its corresponding property:
\begin{itemize}
    \item The pretrained model $\mathbb{P}_G$ is an anchor for sensibleness, attribution, and diversity;
    \item The negative KL-divergence, $-D_{KL}\left(\pi||\mathbb{P}_G\right)$, ensures the policy $\pi$ to stochastically approximate $\mathbb{P}_G$;
    \item The negative cross-entropy, $-H(\pi, \mathbb{P}_G)$, ensures the policy $\pi$ to deterministically and greedily approximate $\mathbb{P}_G$;
    \item The entropy $H(\pi)$ promotes diversity;
    \item The approximate mutual information, $\hat{I}$, promotes a posterior attribution\footnote{We use a priori attribution to denote the attribution metric before sampling and a posterior attribution to denote the attribution metric after sampling.\label{note_a_posterior}} in either discriminative and generative way:
    \begin{itemize}
      \item Discriminative: $\hat{I}(a_t;r_T=1|s_t) = \mathbb{E}_{\pi(\cdot|s_t)}\left(\log\frac{\mathbb{P}_\pi(r_T=1 | s_t, \cdot)}{\mathbb{P}_\pi(r_T=1 | s_t)}\right)$
      \item Generative: $\hat{I}(a_t;e| s_t^-) = \mathbb{E}_{\pi(\cdot|s_t)}\left(\log\frac{\mathbb{P}_G(\cdot|s_t)}{\mathbb{P}_G(\cdot|s_t^-)}\right)$
    \end{itemize}
    This is an approximation since we use different distributions for the expectation and for the inner PMIs.
    \item The dynamic weight $\lambda(s_t, s_t^-)$ opportunistically promotes the term it is paired with when a priori attribution is relevant.
    \item Although $D_{KL}$ in $\lambda(s_t, s_t^-)$ is constructed in a generative way in Theorem \ref{theorem:policy_kl}, the construction in a discriminative way, $D_{KL}(\mathbb{P}_\pi(r_T=1 | s_t, \cdot)||\mathbb{P}_\pi(r_T=1 | s_t))$, is also valid.
\end{itemize}
With informational interpretations, we rewrite the theorems and summarize them in Table \ref{table:theorems}.

\begin{table}
  \caption{Theorems with Informational Interpretations}
  \label{table:theorems}
  \centering
  \begin{tabular}{lll}
    \toprule
    Theorem     & Decoding algorithm     & Action-state value function \\
    \midrule
    \ref{theorem:classifier_guidance}: Classifier guidance & $\left(\frac{\mathbb{P}_\pi(r_T=1 | s_t, a_t)}{\mathbb{P}_\pi(r_T=1 | s_t)}\right)^\lambda \mathbb{P}_G$ & $\lambda\cdot \hat{I}(a_t;r_T=1|s_t) - D_{KL}\left(\pi||\mathbb{P}_G\right)$ \\
    \ref{theorem:classifier_free}: Classifier-free guidance & $\left(\frac{\mathbb{P}_G(a_t|s_t)}{\mathbb{P}_G(a_t|s_t^-)}\right)^\lambda \mathbb{P}_G$ & $\lambda \cdot \hat{I}(a_t;e| s_t^-) - D_{KL}\left(\pi||\mathbb{P}_G\right)$ \\
    \ref{theorem:policy_kl}: \begin{tabular}{@{}l} KL-divergence guided\\temperature sampling \end{tabular}  & $(\mathbb{P}_G)^{\lambda(s_t, s_t^-) + 1}$ & $-\lambda(s_t, s_t^-) H(\pi, \mathbb{P}_G) - D_{KL}\left(\pi||\mathbb{P}_G\right)$ \\
    \ref{theorem:policy_t}: Temperature sampling & $(\mathbb{P}_G)^\frac{1}{T}$ & $-\left(\frac{1}{T}-1\right) H(\pi, \mathbb{P}_G) - D_{KL}\left(\pi||\mathbb{P}_G\right)$ \\[1.5ex]
    \ref{theorem:policy_greedy}: Greedy & $\mathbbm{1}_{a_t = \underset{a}{\arg\max} \; \mathbb{P}_G}$ & $-H(\pi, \mathbb{P}_G)$ \\
    \bottomrule
  \end{tabular}
\end{table}

For classical temperature sampling algorithm in Theorem \ref{theorem:policy_t}, the action-state value function is a weighted sum between the negative KL-divergence, $-D_{KL}(\pi||\mathbb{P}_G)$, and the negative entropy, $-H(\pi)$. As we decrease temperature from $T=1$ towards $T = 0$, the action-state value function starts with purely $-D_{KL}(\pi||\mathbb{P}_G)$ (that binds $\pi$ to $\mathbb{P}_G$) and gradually adds $-H(\pi)$ (that reduces entropy or diversity). Finally when these two terms are equally weighted, they sum up to $-H(\pi, \mathbb{P}_G)$ which defines the greedy algorithm.

While all classical temperature sampling algorithms have a fixed $T$, the work in Theorem \ref{theorem:policy_kl} takes a step further with a dynamic adjustable $T = \left(\lambda(s_t, s_t^-) + 1\right)^{-1}$. By relaxing the constraint of having a constant $T$, it can be opportunistically optimized for different action-state value functions at each decoding step according to its relevance to the evidence $e$. In either cases, there is a clear notion of the tradeoff between the anchor point (delegated by $-D_{KL}(\pi||\mathbb{P}_G)$) and the diversity (delegated by $-H(\pi, \mathbb{P}_G)$ or $-H(\pi)$).

Finally, the works in Theorem \ref{theorem:classifier_guidance} and Theorem \ref{theorem:classifier_free} promotes posterior attributions by adding the approximate mutual information, either guided by the discriminator $D$ or generator $\mathbb{P}_G$. The balance between the approximate mutual information and negative KL-divergence determines how far a distribution can drift from the anchor point for better attributions. There is also a clear notion of the tradeoff between the anchor point (delegated by $-D_{KL}(\pi||\mathbb{P}_G)$) and the attribution (delegated by $I(a_t;r_T=1|s_t)$ or $I(a_t;e| s_t^-)$).

In summary, most of the decoding algorithms compose of a tradeoff between the anchor point and a certain desired property, such as diversity or attributions. Each of the desired property has the respective delegation in the action-state value function space, except for the sensibleness. One mitigation to the sensibleness is to dynamically adjust the weight to the 
delegation to avoid overemphasizing the importance of the property in all decoding steps.

%% file: appendix.tex
\appendix \label{sec:appendix}

\section{Mathematical Preliminaries}

\subsection{Dynamic Programming} \label{sec:dp_definitions}

We follow the convention in~\citep{silver2015} to formulate large language model (LLM) decoder sampling as a dynamic programming (DP) problem.

\begin{definition}[Markov Decision Process]
A Markov Decision Process (MDP) is defined as a tuple of $\langle \mathcal{S}, \mathcal{A}, \mathcal{P}, \mathcal{R}, \gamma \rangle$, where
\begin{itemize}
    \setlength\itemsep{0em}
    \item State space $\mathcal{S}$ defines the set of all possible states that the system may be in;
    \item Action space $\mathcal{A}$ defines the set of all possible actions that an agent can take;
    \item State transition matrix $\mathcal{P}$ defines the probability of the next state given the current state and action, i.e. $\mathcal{P}^a_{ss'}=\mathbb{P}(S_{t+1}=s'|S_t=s,A_t=a)$;
    \item Reward function $\mathcal{R}$ maps the current state and action to the reward incurred at the next time step, i.e. $\mathcal{R}^a_s = \mathbb{E}(R_{t+1} | S_t = s, A_t = a)$
    \item Discount factor $\gamma \in [0, 1]$ penalizes the long-term dependencies.
\end{itemize}
\end{definition}

\begin{definition}[Policy] \label{def:policy}
A policy $\pi$ of an agent is a function that assigns probability distribution to actions for a given state.
\[
    \pi(a|s) = \mathbb{P}(A_t=a | S_t=s)
\]
\end{definition}

\begin{definition}[Return]
The return $G_t$ is the reward-to-go from time step $t$.
\[
    G_t = R_{t+1} + \gamma R_{t+2} + \dots = \sum_{k=0}^\infty \gamma^k R_{t+1+k}.
\]
\end{definition}

\begin{definition}[State Value Function]
The state value function $V_{\pi}(s)$ is the expected return when an agent starts from state $s$ and thereafter acts according to policy $\pi$.
\[
    V_{\pi}(s) = \mathbb{E}(G_t | S_t=s)
\]
\end{definition}

\begin{definition}[Action-State Value Function]
The action-state value function $Q_\pi(s, a)$ is the expected return when an agent starts from state s, takes action a, and thereafter acts according to policy $\pi$.
\[
    Q_\pi(s, a) = \mathbb{E}(G_t | S_t=s, A_t=a)
\]
\end{definition}

\begin{definition}[Optimal Action-State Value Function]
The optimal action-state value function $Q_*(s, a)$ is defined as
\[
    Q_*(s, a) = \max_\pi Q_\pi(s, a)
\]
An optimal policy for $Q_*(s, a)$ is denoted as $\pi_*$.  
\end{definition}

\begin{theorem}[Deterministic Optimal Policy]
A deterministic optimal policy can be constructed by
\[
    \pi_*(a|s) = \begin{cases}1 & \text{if \;\;} a = \underset{a \in \mathcal{A}}{\arg\max}\, Q_*(s,a) \\ 0 & \text{otherwise.} \end{cases}
\]
\end{theorem}

\subsection{Pointwise Mutual Information (PMI)} \label{sec:define_pmi}

\begin{definition}[Pointwise Mutual Information]
The pointwise mutual information of a pair of discrete distributions $x$ and $y$ is defined as
\[
    \pmi(x, y) = \log\frac{\mathbb{P}(x, y)}{\mathbb{P}(x)\mathbb{P}(y)} = \log\frac{\mathbb{P}(x| y)}{\mathbb{P}(x)} = \log\frac{\mathbb{P}(y| x)}{\mathbb{P}(y)}.
\]
\end{definition}

\section{Proof of Theorems}  \label{sec:proof_of_theorem}

\subsection{Proof of Theorem \ref{theorem:classifier_guidance}} \label{sec:proof_of_theorem_policy_discriminator}

Following the derivation in~\citep{rafailov2023direct} Appendix A.1, we have the explicit optimal policy
\begin{align*}
    \pi_* &= \frac{\left(\frac{\mathbb{P}_\pi(r_T=1 | s_t, a_t)}{\mathbb{P}_\pi(r_T=1 | s_t)}\right)^\lambda \mathbb{P}_G(a_t|s_t)}{K(s_t)} \\
    &\propto \left(\frac{\mathbb{P}_\pi(r_T=1 | s_t, a_t)}{\mathbb{P}_\pi(r_T=1 | s_t)}\right)^\lambda \mathbb{P}_G(a_t|s_t)
\end{align*}
where $K(s_t) = \sum_{a \in \pi(\cdot|s_t)}\left(\frac{\mathbb{P}_\pi(r_T=1 | s_t, a)}{\mathbb{P}_\pi(r_T=1 | s_t)}\right)^\lambda \mathbb{P}_G(a|s_t)$ be a normalization factor.

\subsection{Proof of Theorem \ref{theorem:policy_kl}}

We can rewrite the action-state value function as follows:
\begin{align*}
    \pi_*(a_t|s_t) &= \underset{\pi}{\arg\max}\left\{
    -(\lambda(s_t, s_t^-)+1) D_{KL}\left(\pi||\mathbb{P}_G\right) - \lambda(s_t, s_t^-) H\left(\pi\right) \right\} \\
    &= \underset{\pi}{\arg\max}\left\{
    -\lambda(s_t, s_t^-) (H(\pi) + D_{KL}(\pi||\mathbb{P}_G)) - D_{KL}\left(\pi||\mathbb{P}_G\right) \right\} \\
    &= \underset{\pi}{\arg\max}\left\{
    -\lambda(s_t, s_t^-) H\left(\pi, \mathbb{P}_G\right) - D_{KL}\left(\pi||\mathbb{P}_G\right) \right\} \\
    &= \underset{\pi}{\arg\max}\left\{
    \lambda(s_t, s_t^-)\,\mathbb{E}_{\pi(\cdot|s_t)}\log\mathbb{P}_G(\cdot|s_t) - D_{KL}\left(\pi||\mathbb{P}_G\right) \right\} \\
    &= \underset{\pi}{\arg\max}\left\{
    \mathbb{E}_{\pi(\cdot|s_t)}\log\mathbb{P}_G(\cdot|s_t)^{\lambda(s_t, s_t^-)} - D_{KL}\left(\pi||\mathbb{P}_G\right) \right\}    
\end{align*}

Deriving the explicit optimal policy from the action-state function follows that in~\citep{rafailov2023direct} Appendix A.1.

%% file: main.bbl
\begin{thebibliography}{25}
\providecommand{\natexlab}[1]{#1}
\providecommand{\url}[1]{\texttt{#1}}
\expandafter\ifx\csname urlstyle\endcsname\relax
  \providecommand{\doi}[1]{doi: #1}\else
  \providecommand{\doi}{doi: \begingroup \urlstyle{rm}\Url}\fi

\bibitem[Aksitov et~al.(2023)Aksitov, Chang, Reitter, Shakeri, and
  Sung]{aksitov2023characterizing}
Renat Aksitov, Chung-Ching Chang, David Reitter, Siamak Shakeri, and Yunhsuan
  Sung.
\newblock Characterizing attribution and fluency tradeoffs for
  retrieval-augmented large language models.
\newblock \emph{arXiv preprint arXiv:2302.05578}, 2023.

\bibitem[Anil et~al.(2023)Anil, Dai, Firat, Johnson, Lepikhin, Passos, Shakeri,
  Taropa, Bailey, Chen, Chu, Clark, Shafey, Huang, Meier-Hellstern, Mishra,
  Moreira, Omernick, Robinson, Ruder, Tay, Xiao, Xu, Zhang, Abrego, Ahn,
  Austin, Barham, Botha, Bradbury, Brahma, Brooks, Catasta, Cheng, Cherry,
  Choquette-Choo, Chowdhery, Crepy, Dave, Dehghani, Dev, Devlin, Díaz, Du,
  Dyer, Feinberg, Feng, Fienber, Freitag, Garcia, Gehrmann, Gonzalez, Gur-Ari,
  Hand, Hashemi, Hou, Howland, Hu, Hui, Hurwitz, Isard, Ittycheriah, Jagielski,
  Jia, Kenealy, Krikun, Kudugunta, Lan, Lee, Lee, Li, Li, Li, Li, Li, Lim, Lin,
  Liu, Liu, Maggioni, Mahendru, Maynez, Misra, Moussalem, Nado, Nham, Ni,
  Nystrom, Parrish, Pellat, Polacek, Polozov, Pope, Qiao, Reif, Richter, Riley,
  Ros, Roy, Saeta, Samuel, Shelby, Slone, Smilkov, So, Sohn, Tokumine, Valter,
  Vasudevan, Vodrahalli, Wang, Wang, Wang, Wang, Wieting, Wu, Xu, Xu, Xue, Yin,
  Yu, Zhang, Zheng, Zheng, Zhou, Zhou, Petrov, and Wu]{anil2023palm}
Rohan Anil, Andrew~M. Dai, Orhan Firat, Melvin Johnson, Dmitry Lepikhin,
  Alexandre Passos, Siamak Shakeri, Emanuel Taropa, Paige Bailey, Zhifeng Chen,
  Eric Chu, Jonathan~H. Clark, Laurent~El Shafey, Yanping Huang, Kathy
  Meier-Hellstern, Gaurav Mishra, Erica Moreira, Mark Omernick, Kevin Robinson,
  Sebastian Ruder, Yi~Tay, Kefan Xiao, Yuanzhong Xu, Yujing Zhang,
  Gustavo~Hernandez Abrego, Junwhan Ahn, Jacob Austin, Paul Barham, Jan Botha,
  James Bradbury, Siddhartha Brahma, Kevin Brooks, Michele Catasta, Yong Cheng,
  Colin Cherry, Christopher~A. Choquette-Choo, Aakanksha Chowdhery, Clément
  Crepy, Shachi Dave, Mostafa Dehghani, Sunipa Dev, Jacob Devlin, Mark Díaz,
  Nan Du, Ethan Dyer, Vlad Feinberg, Fangxiaoyu Feng, Vlad Fienber, Markus
  Freitag, Xavier Garcia, Sebastian Gehrmann, Lucas Gonzalez, Guy Gur-Ari,
  Steven Hand, Hadi Hashemi, Le~Hou, Joshua Howland, Andrea Hu, Jeffrey Hui,
  Jeremy Hurwitz, Michael Isard, Abe Ittycheriah, Matthew Jagielski, Wenhao
  Jia, Kathleen Kenealy, Maxim Krikun, Sneha Kudugunta, Chang Lan, Katherine
  Lee, Benjamin Lee, Eric Li, Music Li, Wei Li, YaGuang Li, Jian Li, Hyeontaek
  Lim, Hanzhao Lin, Zhongtao Liu, Frederick Liu, Marcello Maggioni, Aroma
  Mahendru, Joshua Maynez, Vedant Misra, Maysam Moussalem, Zachary Nado, John
  Nham, Eric Ni, Andrew Nystrom, Alicia Parrish, Marie Pellat, Martin Polacek,
  Alex Polozov, Reiner Pope, Siyuan Qiao, Emily Reif, Bryan Richter, Parker
  Riley, Alex~Castro Ros, Aurko Roy, Brennan Saeta, Rajkumar Samuel, Renee
  Shelby, Ambrose Slone, Daniel Smilkov, David~R. So, Daniel Sohn, Simon
  Tokumine, Dasha Valter, Vijay Vasudevan, Kiran Vodrahalli, Xuezhi Wang,
  Pidong Wang, Zirui Wang, Tao Wang, John Wieting, Yuhuai Wu, Kelvin Xu, Yunhan
  Xu, Linting Xue, Pengcheng Yin, Jiahui Yu, Qiao Zhang, Steven Zheng,
  Ce~Zheng, Weikang Zhou, Denny Zhou, Slav Petrov, and Yonghui Wu.
\newblock Palm 2 technical report, 2023.

\bibitem[Chang et~al.(2023)Chang, Reitter, Aksitov, and
  Sung]{chang2023kldivergence}
Chung-Ching Chang, David Reitter, Renat Aksitov, and Yun-Hsuan Sung.
\newblock {KL}-divergence guided temperature sampling, 2023.

\bibitem[Dhariwal and Nichol(2021)]{dhariwal2021diffusion}
Prafulla Dhariwal and Alex Nichol.
\newblock Diffusion models beat gans on image synthesis, 2021.

\bibitem[Ho and Salimans(2022)]{ho2022classifierfree}
Jonathan Ho and Tim Salimans.
\newblock Classifier-free diffusion guidance, 2022.

\bibitem[Hu et~al.(2021)Hu, Shen, Wallis, Allen-Zhu, Li, Wang, Wang, and
  Chen]{hu2021lora}
Edward~J. Hu, Yelong Shen, Phillip Wallis, Zeyuan Allen-Zhu, Yuanzhi Li, Shean
  Wang, Lu~Wang, and Weizhu Chen.
\newblock Lora: Low-rank adaptation of large language models, 2021.

\bibitem[Huang et~al.(2023)Huang, Liu, Lin, Pang, Du, and
  Lin]{huang2023lorahub}
Chengsong Huang, Qian Liu, Bill~Yuchen Lin, Tianyu Pang, Chao Du, and Min Lin.
\newblock Lorahub: Efficient cross-task generalization via dynamic lora
  composition, 2023.

\bibitem[Kim et~al.(2022)Kim, Lee, Yoo, Park, Lee, and
  Jung]{kim2022criticguided}
Minbeom Kim, Hwanhee Lee, Kang~Min Yoo, Joonsuk Park, Hwaran Lee, and Kyomin
  Jung.
\newblock Critic-guided decoding for controlled text generation, 2022.

\bibitem[Krause et~al.(2020)Krause, Gotmare, McCann, Keskar, Joty, Socher, and
  Rajani]{krause2020gedi}
Ben Krause, Akhilesh~Deepak Gotmare, Bryan McCann, Nitish~Shirish Keskar,
  Shafiq Joty, Richard Socher, and Nazneen~Fatema Rajani.
\newblock Gedi: Generative discriminator guided sequence generation, 2020.

\bibitem[Lester et~al.(2021)Lester, Al-Rfou, and Constant]{lester2021power}
Brian Lester, Rami Al-Rfou, and Noah Constant.
\newblock The power of scale for parameter-efficient prompt tuning, 2021.

\bibitem[Li and Liang(2021)]{li2021prefixtuning}
Xiang~Lisa Li and Percy Liang.
\newblock Prefix-tuning: Optimizing continuous prompts for generation, 2021.

\bibitem[Li et~al.(2022)Li, Thickstun, Gulrajani, Liang, and
  Hashimoto]{li2022diffusionlm}
Xiang~Lisa Li, John Thickstun, Ishaan Gulrajani, Percy Liang, and Tatsunori~B.
  Hashimoto.
\newblock Diffusion-lm improves controllable text generation, 2022.

\bibitem[Liu et~al.(2022)Liu, Tam, Muqeeth, Mohta, Huang, Bansal, and
  Raffel]{liu2022fewshot}
Haokun Liu, Derek Tam, Mohammed Muqeeth, Jay Mohta, Tenghao Huang, Mohit
  Bansal, and Colin Raffel.
\newblock Few-shot parameter-efficient fine-tuning is better and cheaper than
  in-context learning, 2022.

\bibitem[Mudgal et~al.(2023)Mudgal, Lee, Ganapathy, Li, Wang, Huang, Chen,
  Cheng, Collins, Strohman, Chen, Beutel, and Beirami]{mudgal2023controlled}
Sidharth Mudgal, Jong Lee, Harish Ganapathy, YaGuang Li, Tao Wang, Yanping
  Huang, Zhifeng Chen, Heng-Tze Cheng, Michael Collins, Trevor Strohman, Jilin
  Chen, Alex Beutel, and Ahmad Beirami.
\newblock Controlled decoding from language models, 2023.

\bibitem[OpenAI(2023)]{openai2023gpt4}
OpenAI.
\newblock {GPT}-4 technical report, 2023.

\bibitem[Ouyang et~al.(2022)Ouyang, Wu, Jiang, Almeida, Wainwright, Mishkin,
  Zhang, Agarwal, Slama, Ray, Schulman, Hilton, Kelton, Miller, Simens, Askell,
  Welinder, Christiano, Leike, and Lowe]{ouyang2022training}
Long Ouyang, Jeff Wu, Xu~Jiang, Diogo Almeida, Carroll~L. Wainwright, Pamela
  Mishkin, Chong Zhang, Sandhini Agarwal, Katarina Slama, Alex Ray, John
  Schulman, Jacob Hilton, Fraser Kelton, Luke Miller, Maddie Simens, Amanda
  Askell, Peter Welinder, Paul Christiano, Jan Leike, and Ryan Lowe.
\newblock Training language models to follow instructions with human feedback,
  2022.

\bibitem[Pei et~al.(2023)Pei, Yang, and Klein]{pei2023preadd}
Jonathan Pei, Kevin Yang, and Dan Klein.
\newblock Preadd: Prefix-adaptive decoding for controlled text generation,
  2023.

\bibitem[Radford et~al.(2019)Radford, Wu, Child, Luan, Amodei, Sutskever,
  et~al.]{radford2019language}
Alec Radford, Jeffrey Wu, Rewon Child, David Luan, Dario Amodei, Ilya
  Sutskever, et~al.
\newblock Language models are unsupervised multitask learners.
\newblock \emph{OpenAI blog}, 1\penalty0 (8):\penalty0 9, 2019.

\bibitem[Rafailov et~al.(2023)Rafailov, Sharma, Mitchell, Ermon, Manning, and
  Finn]{rafailov2023direct}
Rafael Rafailov, Archit Sharma, Eric Mitchell, Stefano Ermon, Christopher~D.
  Manning, and Chelsea Finn.
\newblock Direct preference optimization: Your language model is secretly a
  reward model, 2023.

\bibitem[Saharia et~al.(2022)Saharia, Chan, Saxena, Li, Whang, Denton,
  Ghasemipour, Ayan, Mahdavi, Lopes, Salimans, Ho, Fleet, and
  Norouzi]{saharia2022photorealistic}
Chitwan Saharia, William Chan, Saurabh Saxena, Lala Li, Jay Whang, Emily
  Denton, Seyed Kamyar~Seyed Ghasemipour, Burcu~Karagol Ayan, S.~Sara Mahdavi,
  Rapha~Gontijo Lopes, Tim Salimans, Jonathan Ho, David~J Fleet, and Mohammad
  Norouzi.
\newblock Photorealistic text-to-image diffusion models with deep language
  understanding, 2022.

\bibitem[Santacroce et~al.(2023)Santacroce, Lu, Yu, Li, and
  Shen]{santacroce2023efficient}
Michael Santacroce, Yadong Lu, Han Yu, Yuanzhi Li, and Yelong Shen.
\newblock Efficient rlhf: Reducing the memory usage of ppo, 2023.

\bibitem[Shi et~al.(2023)Shi, Han, Lewis, Tsvetkov, Zettlemoyer, and tau
  Yih]{shi2023trusting}
Weijia Shi, Xiaochuang Han, Mike Lewis, Yulia Tsvetkov, Luke Zettlemoyer, and
  Scott~Wen tau Yih.
\newblock Trusting your evidence: Hallucinate less with context-aware decoding,
  2023.

\bibitem[Silver(2020)]{silver2015}
David Silver.
\newblock {Lecture 2: Markov Decision Processes} explores markov processes
  including reward processes, decision processes and extensions.
\newblock \url{https://www.davidsilver.uk/wp-content/uploads/2020/03/MDP.pdf},
  2020.
\newblock Accessed: 2023-08-04.

\bibitem[Yang and Klein(2021)]{Yang_2021}
Kevin Yang and Dan Klein.
\newblock {FUDGE}: Controlled text generation with future discriminators.
\newblock In \emph{Proceedings of the 2021 Conference of the North American
  Chapter of the Association for Computational Linguistics: Human Language
  Technologies}. Association for Computational Linguistics, 2021.
\newblock \doi{10.18653/v1/2021.naacl-main.276}.
\newblock URL \url{https://doi.org/10.18653%2Fv1%2F2021.naacl-main.276}.

\bibitem[Yu et~al.(2022)Yu, Xu, Koh, Luong, Baid, Wang, Vasudevan, Ku, Yang,
  Ayan, Hutchinson, Han, Parekh, Li, Zhang, Baldridge, and Wu]{yu2022scaling}
Jiahui Yu, Yuanzhong Xu, Jing~Yu Koh, Thang Luong, Gunjan Baid, Zirui Wang,
  Vijay Vasudevan, Alexander Ku, Yinfei Yang, Burcu~Karagol Ayan, Ben
  Hutchinson, Wei Han, Zarana Parekh, Xin Li, Han Zhang, Jason Baldridge, and
  Yonghui Wu.
\newblock Scaling autoregressive models for content-rich text-to-image
  generation, 2022.

\end{thebibliography}
